\documentclass[letterpaper,11pt]{article}

\usepackage{amsmath}
\usepackage{amsfonts}
\usepackage{amssymb}
\usepackage{amsthm}
\usepackage{url,ifthen}
\usepackage{enumitem}
\usepackage{srcltx}
\usepackage{multirow}
\usepackage{boxedminipage}
\usepackage[margin=1.1in]{geometry}
\usepackage{nicefrac}
\usepackage{xspace}
\usepackage{graphicx}
\usepackage{color}
\definecolor{DarkGreen}{rgb}{0.1,0.5,0.1}
\definecolor{DarkRed}{rgb}{0.5,0.1,0.1}
\definecolor{DarkBlue}{rgb}{0.1,0.1,0.5}
\definecolor{Gray}{rgb}{0.2,0.2,0.2}

\usepackage{listings}
\lstdefinestyle{mystyle}{
    commentstyle=\color{DarkBlue},
    keywordstyle=\color{DarkRed},
    numberstyle=\tiny\color{Gray},
    stringstyle=\color{DarkGreen},
    basicstyle=\footnotesize,
    breakatwhitespace=false,         
    breaklines=true,                 
    captionpos=b,                    
    keepspaces=true,                 
    numbers=left,                    
    numbersep=5pt,                  
    showspaces=false,                
    showstringspaces=false,
    showtabs=false,                  
    tabsize=2
}
\usepackage[small]{caption}
\usepackage[pdftex]{hyperref}
\hypersetup{
    unicode=false,          
    pdftoolbar=true,        
    pdfmenubar=true,        
    pdffitwindow=false,      
    pdfnewwindow=true,      
    colorlinks=true,       
    linkcolor=DarkBlue,          
    citecolor=DarkGreen,        
    filecolor=DarkRed,      
    urlcolor=DarkBlue,          
    pdftitle={},
    pdfauthor={Moritz Hardt},
}

\def\draft{1}

\def\submit{0}

\ifnum\draft=1 
    \def\ShowAuthNotes{1}
\else
    \def\ShowAuthNotes{0}
\fi

\ifnum\submit=1
\newcommand{\forsubmit}[1]{#1}
\newcommand{\forreals}[1]{}
\else
\newcommand{\forreals}[1]{#1}
\newcommand{\forsubmit}[1]{}
\fi

\ifnum\ShowAuthNotes=1
\newcommand{\authnote}[2]{{ \footnotesize \bf{\color{DarkRed}[#1's Note:
{\color{DarkBlue}#2}]}}}
\else
\newcommand{\authnote}[2]{}
\fi

\newtheorem{theorem}{Theorem}[section]

\newtheorem{lemma}[theorem]{Lemma}
\newtheorem{corollary}[theorem]{Corollary}

\newtheorem{claim}[theorem]{Claim}

\theoremstyle{definition}
\newtheorem{definition}[theorem]{Definition}

\newcommand{\chapterref}[1]{\hyperref[ch:#1]{Chapter~\ref{ch:#1}}}
\newcommand{\claimlabel}[1]{\label{claim:#1}}
\newcommand{\claimref}[1]{\hyperref[claim:#1]{Claim~\ref{claim:#1}}}
\newcommand{\corollarylabel}[1]{\label{cor:#1}}
\newcommand{\corollaryref}[1]{\hyperref[cor:#1]{Corollary~\ref{cor:#1}}}
\newcommand{\definitionlabel}[1]{\label{def:#1}}
\newcommand{\definitionref}[1]{\hyperref[def:#1]{Definition~\ref{def:#1}}}
\newcommand{\equationlabel}[1]{\label{eq:#1}}
\newcommand{\equationref}[1]{\hyperref[eq:#1]{Equation~\ref{eq:#1}}}

\newcommand{\factref}[1]{\hyperref[fact:#1]{Fact~\ref{fact:#1}}}
\newcommand{\figurelabel}[1]{\label{fig:#1}}
\newcommand{\figureref}[1]{\hyperref[fig:#1]{Figure~\ref{fig:#1}}}

\newcommand{\tableref}[1]{\hyperref[tab:#1]{Table~\ref{tab:#1}}}

\newcommand{\itemref}[1]{\hyperref[item:#1]{Item~(\ref{item:#1})}}
\newcommand{\lemmalabel}[1]{\label{lem:#1}}
\newcommand{\lemmaref}[1]{\hyperref[lem:#1]{Lemma~\ref{lem:#1}}}

\newcommand{\propref}[1]{\hyperref[prop:#1]{Proposition~\ref{prop:#1}}}

\newcommand{\propositionref}[1]{\hyperref[prop:#1]{Proposition~\ref{prop:#1}}}

\newcommand{\remarkref}[1]{\hyperref[rem:#1]{Remark~\ref{rem:#1}}}
\newcommand{\sectionlabel}[1]{\label{sec:#1}}
\newcommand{\sectionref}[1]{\hyperref[sec:#1]{Section~\ref{sec:#1}}}
\newcommand{\theoremlabel}[1]{\label{thm:#1}}
\newcommand{\theoremref}[1]{\hyperref[thm:#1]{Theorem~\ref{thm:#1}}}

\usepackage[utf8]{inputenc}
\usepackage[T1]{fontenc}
\usepackage{kpfonts}
\usepackage{microtype}

\newcommand{\Esymb}{\mathbb{E}}
\newcommand{\Psymb}{\mathbb{P}}

\DeclareMathOperator*{\E}{\Esymb}

\DeclareMathOperator*{\ProbOp}{\Psymb r}
\renewcommand{\Pr}{\ProbOp}

\newcommand{\mper}{\,.}

\renewcommand{\hat}{\widehat}

\newcommand{\cA}{{\cal A}}

\newcommand{\cD}{{\cal D}}

\newcommand{\defeq}{\stackrel{\small \mathrm{def}}{=}}

\renewcommand{\le}{\leqslant}

\renewcommand{\ge}{\geqslant}

\newcommand{\Set}[1]{\left\{#1\right\}}

\usepackage{bm}

\newcommand{\ignore}[1]{}

\renewcommand{\epsilon}{\varepsilon}

\newcommand{\Laplace}{\mathrm{Lap}}
\newcommand{\Lap}{\Laplace}

\newcommand{\remove}[1]{}

\newenvironment{itm}
{\begin{itemize}[noitemsep,topsep=0pt,parsep=0pt,partopsep=0pt]}
{\end{itemize}}
\newenvironment{enum}
{\begin{enumerate}[noitemsep,topsep=0pt,parsep=0pt,partopsep=0pt]}
{\end{enumerate}}

\newcommand{\lberr}{\mathrm{lberr}}

\title{Climbing a shaky ladder:\\ Better adaptive risk estimation}
\author{Moritz Hardt}
\begin{document}
\maketitle
\begin{abstract}
We revisit the \emph{leaderboard problem} introduced by Blum and Hardt (2015)
in an effort to reduce overfitting in machine learning benchmarks. We show
that a randomized version of their Ladder algorithm achieves leaderboard error
$O(1/n^{0.4})$ compared with the previous best rate of $O(1/n^{1/3}).$

Short of proving that our algorithm is optimal, we point out  a major obstacle
toward further progress. Specifically, any improvement to our upper bound
would lead to asymptotic improvements in the general adaptive estimation
setting as have remained elusive in recent years. This connection also
directly leads to lower bounds for specific classes of algorithms. In
particular, we exhibit a new attack on the leaderboard algorithm that both
theoretically and empirically distinguishes between our algorithm and previous
leaderboard algorithms.  
\end{abstract}

\section{Introduction}

Machine learning benchmarks across industry and science are largely based on
the simple mechanism of a holdout set. Participants repeatedly evaluate their
models on the holdout set and use the feedback to improve their models. This
feedback loop has become the de facto experimental paradigm in machine
learning. What is concerning is that the analyst uses the holdout in a
sequential and adaptive manner, thus creating dependencies between the model
to be evaluated and the holdout data. The lack of independence between model
and holdout data is what invalidates classical confidence bounds for the
holdout setting. This insight was articulated in sequence of papers on what is
now called \emph{adaptive data
analysis}~\cite{DFHPRR15stoc,HardtU14,DFHPRR15science}. In a general
formulation, adaptive data analysis can be thought of as an interaction
between an algorithm that holds the sample, and an analyst that repeatedly
asks queries about the data, such as ``What is the loss of this model on the
underlying population?'' 

In its general formulation, adaptive data analysis runs into strong
computational lower bounds. Under computational hardness assumptions, no
computationally efficient algorithm working with~$n$ samples can preserve even
mild statistical validity on more than $n^2$ queries
~\cite{HardtU14,SteinkeU14}.  This stands in sharp contrast to the
non-adaptive setting where the error bounds deteriorate logarithmically with
the number of queries~$k.$

Circumventing these lower bounds, Blum and Hardt~\cite{BH15} introduced a
simpler setting that allowed for much better guarantees. The key idea is that
oftentimes it's sufficient to find the best model out of a sequence of
adaptively chosen models, or to keep a ranking of some of the models. This is
the relevant task in machine learning benchmarks, competitions, and
hyperparameter tuning. Even adaptive early stopping can be posed as an
instance of this problem. Within this framework, there's a particularly simple
and efficient algorithm called the \emph{Ladder} algorithm. The algorithm
maintains an internal threshold. Whenever a given model exceeds the previous
quality threshold by a significant amount, the algorithm updates the threshold
and provides the analyst with feedback about the quality of the model. If the
model did not exceed the threshold, the analyst receives no feedback at all.

The Ladder algorithm maintains the \emph{risk} of the best model (with respect
to a bounded loss function) on a sequence of $k$ adaptively chosen models up
to an additive error of $O(\log(kn)^{1/3}/n^{1/3}).$ This type of guarantee is
called \emph{leaderboard error}, since it does not require an accurate
estimate for all models, but only the best performing one at any point in
time. While this bound features a logarithmic dependence on $k,$ the rate in
terms of~$n$ falls short of the non-adaptive bound~$O(\sqrt{\log(k)/n}).$

\subsection{Our contributions}

We narrow the gap between existing upper and lower bounds. Our first result is
a randomized variant of the Ladder algorithm, called \emph{Shaky Ladder} that
achieves leaderboard error $O(1/n^{0.4}).$

\begin{theorem}[Informal version of \theoremref{ub}]
On $n$ samples and $k$ adaptively chosen models, the Shaky Ladder achieves
with high probability leaderboard error 
\[
O \left(\frac{\log(k)^{2/5}\log(kn)^{1/5}}{n^{2/5}}\right).
\]
\end{theorem}

The algorithm is based on analyzing noise addition via differential privacy,
in particular, the so-called \emph{sparse vector technique} as described
in~\cite{DworkR14}. We combine this analysis with powerful adaptive
generalization bounds for differential privacy, where it is important to use
the recently improved bound of Bassily et al.~\cite{BassilyNSSSU16}. The
earlier bound due to Dwork et al.~\cite{DFHPRR15stoc} would not suffice to
give any improvement over the Ladder algorithm that achieved leaderboard
error~$O(\log(kn)^{1/3}/n^{1/3}).$ 

Our upper bound falls short of the information-theoretic lower bound of
$\Omega(\sqrt{\log(k)/n})$ that holds even in the non-adaptive estimation
setting. Intuition from online learning and the literature on bandit
algorithms suggest that either the exponent $1/3$ or the exponent $1/2$ could
be a natural answer. Surprisingly, our result shows that a natural algorithm
achieves the unusual rate of $1/n^{0.4}.$ Moreover, we show that going beyond
this rate will likely require powerful new techniques.

In order to make this point, we develop a new connection between leaderboard
and the general adaptive estimation setting. Specifically, we show that any
accurate leaderboard algorithm for sufficiently many queries readily implies a
general adaptive estimator (formally introduced in \sectionref{reduction}) 
for a smaller number of queries.

\begin{theorem}[Informal version of \theoremref{reduction}]
Suppose there exists a leaderboard algorithm $\cal A$ that is
$(\alpha/2)$-accurate on $n$ samples and $1/\alpha^2$ models. Then, there
exists a general adaptive estimator~$\cal B$ that is $\alpha$-accurate on
$k=1/3\alpha$ queries.
\end{theorem}

In the regime where $k\le n$, the best current upper bound is $\alpha=\tilde
O(k^{1/4}/\sqrt{n}).$ For $k=n^{0.4},$ this bound simplifies to $\tilde
O(1/n^{0.4})$ and thus coincides with what would follow from our theorem. This
is no coincidence since the bounds are proved using the same techniques. What
is new, however, is that any further improvement in leaderboard accuracy over
our result would directly improve on the best known bounds in the general
adaptive estimation setting. In particular, a leaderboard upper bound of
$O(\sqrt{\log(k)/n}),$ as is currently not ruled out, would lead to a general
adaptive estimator for nearly $\sqrt{n}$ queries and accuracy $\tilde
O(1/\sqrt{n}).$ Going to the natural statistical rate of $O(1/\sqrt{n})$ has
remained elusive in the general adaptive estimation setting for any $k\ge n^c$
with $c>0$. What our result shows is that this task is no easier in the
leaderboard setting. It's worth noting that there are lower bounds in special
cases, e.g.,~\cite{RussoZ15,WangLF16}.

We use \theoremref{reduction} to prove a lower bound against a natural class
of leaderboard algorithms that we call \emph{faithful}. Intuitively, speaking
when faithful algorithms return feedback, the feedback is close to the
\emph{empirical risk} of the submitted model with high probability. This class
of algorithms includes both the Ladder algorithm and it's heuristic
counterpart the \emph{parameter-free Ladder}. While those algorithms are
deterministic, faithful algorithms may also be randomized.

\begin{theorem}[Informal version of \corollaryref{faithful}]
No faithful algorithm can achieve leaderboard error~$o(n^{-1/3}).$
\end{theorem}

In particular, this theorem separates our algorithm from earlier work. In
\sectionref{attack}, we illustrate this separation with a practical attack
that causes a major bias in the Ladder algorithm, while being ineffective
against our algorithm. 

Beyond the related work already discussed, Neto et al.~\cite{Neto16} proposed
a number of heuristic leaderboard algorithms based on the idea of replacing
the holdout estimate by Bootstrap estimates. In practice, this results in
noise addition that can be helpful. However, these algorithms do not come with
theoretical bound on the leaderboard error better than the Ladder.

\subsection{Preliminaries}
Let $X$ be a data domain and $Y$ be a finite set of class labels, e.g.,
$X=\mathbb{R}^d$ and $Y=\{0,1\}.$ A \emph{loss function} is
a mapping~$\ell\colon Y\times Y\to[0,1]$ and a model is a
mapping $f\colon X \to Y.$
A standard loss function is the $0/1$-loss defined as $\ell_{01}(y,y')=1$ if
$y\ne y'$ and $0$ otherwise. Throughout this paper we assume that $\ell$ is a
loss function with bounded range.
We assume that we are given a sample $S = \{(x_1,y_1),\dots,(x_n,y_n)\}$ drawn
i.i.d. from an unknown distribution $\cD$ over $X\times Y.$ The \emph{risk} of a
model~$f$ is defined as its expected loss on the unknown distribution
$R_\cD(f) \defeq \E_{(x,y)\sim\cD}\left[\ell(f(x),y))\right]\mper$
The \emph{empirical risk} is the standard way of estimating risk from a sample.
$R_S(f) \defeq \frac1n \sum_{i=1}^n \ell(f(x_i),y_i)\mper$

\paragraph{Adaptive risk estimation.}
Given a sequence of models $f_1,\dots,f_k$ and a finite sample~$S$ of size
$n,$ a fundamental estimation problem is to compute estimates $R_1,\dots,R_k$
of the risk of each model. Classically, this is done via the empirical
risk. Applying Hoeffding's bound to each empirical risk estimate, and taking a
union bound over all functions, reveals that the largest deviation of any such
estimate is bounded by $O(\sqrt{\log(k)/n}).$ This is the estimation error we
expect to see in the standard \emph{non-adaptive} setting.

In the \emph{adaptive} estimation setting, we assume that the model $f_t$
may be chosen by an analyst as a function of previously observed estimates and previously chosen models. Formally, there exists a mapping $\cA$ such that for all $t\in[k],$ the mapping $\cA$ returns a function $f_t = \cA(f_1,R_1,\dots,f_{t-1},R_{t-1})$ from all previously observed information.  We will assume for simplicity that the analyst $\cA$ is a deterministic algorithm. The tuple $(f_1,R_1,\dots,f_{t-1},R_{t-1})$ is nevertheless a random variable due to the random sample used to compute the estimates, as well possibly additional randomness introduced in the estimates.
A natural notion of estimation error in the adaptive setting is the maximum
error of any of the estimates, i.e., $\max_{1\le t\le k}\left|R_\cD(f_i) -
R_t\right|.$ Unfortunately, lower bounds~\cite{HardtU14,SteinkeU14} show that
no computationally efficient estimator can achieve maximum error $o(1)$ on
more than $n^{2+o(1)}$ adaptively chosen functions (under a standard hardness
assumption). 

\paragraph{Leaderboard error.} 
Blum and Hardt~\cite{BH15} introduced a weaker notion of estimation error
called \emph{leaderboard error}. Informally speaking, leaderboard error asks
us to maintain a good estimate of the best (lowest risk) model seen so
far, but does not require an accurate estimate for all models that we
encounter.
\begin{definition}[Leaderboard error]
Given an adaptively chosen sequence of models~$f_1,\dots,f_k,$
we define the \emph{leaderboard error} of estimates $R_1,\dots,R_k$
as
\begin{equation}\equationlabel{lberr}
\textstyle
\lberr(R_1,\dots,R_k)
\defeq\max_{1\le t\le k}\left|\min_{1\le i\le t} R_\cD(f_i) - R_t\right|
\end{equation}
\end{definition}

\section{The Shaky Ladder algorithm}

We introduce an algorithm called Shaky Ladder that achieves small leaderboard
accuracy. The algorithm is very simple. For each given function, it compares
the empirical risk of the function to the previously smallest empirical risk
plus some noise variables. If the estimate is below the previous best by some
margin, it releases the estimate plus noise and updates the best estimate.
Importantly, if the estimate is not smaller by a margin, the algorithm
releases the previous best risk (rather than the new estimate). A formal
description follows in \figureref{ladder}. For simplicity we assume we know an
upper bound~$k$ on the total number of rounds.

\begin{figure}[h]
\setlength{\fboxsep}{2mm}
\begin{center}
\begin{boxedminipage}{\textwidth}

\noindent {\bf Input:} Data sets $S$ with $n=|S|,$
step size~$\lambda>0,$ parameters $\epsilon\in(0,1/3),\delta\in(0,\epsilon/4).$
Let $\sigma=\sqrt{\log(1/\delta)}/(\epsilon n).$

\noindent {\bf Algorithm:}

\begin{itm}
\item Assign initial estimate $R_0\leftarrow 1.$
\item Sample noise $\xi\leftarrow\Lap(\sigma).$
\item {\bf For each} round $t \leftarrow 1,2 \ldots k:$
\begin{enum}
\item Receive function $f_t\colon X\to Y$
\item Sample noise variables $\xi_t, \xi_t', \xi_t''\sim\Lap(\sigma)$
independently.
\item {\bf If} $R_S(f_t) + \xi_t < R_{t-1} - \lambda + \xi$
\begin{enum}
\item $R_t\leftarrow R_{S}(f_t) + \xi_t'$
\item $\xi\leftarrow \xi_t''.$
\end{enum}
\item {\bf Else} assign $R_t \leftarrow R_{t-1}.$
\item {\bf Output} $R_t$
\end{enum}
\end{itm}
\end{boxedminipage}
\end{center}
\vspace{-3mm}
\caption{
\figurelabel{ladder} The Shaky Ladder algorithm.
}
\end{figure}

\paragraph{Parameter settings.}
We introduce a new parameter $\beta>0$ for the failure probability of our
algorithm. For the purpose of our analysis we fix the parameters as follows:
\begin{equation}
\equationlabel{parameters}
\delta = \frac{\beta}{kn}
\qquad
\epsilon = \left(\frac{\log(k/\beta)\sqrt{\log(1/\delta)}}{n}\right)^{3/5}
\qquad
\lambda = 4\log(4k/\beta)\sigma
\end{equation}
With these settings all parameters are frozen with the one expection of~$\beta.$
The settings are optimized to prove the theorem, and do not necessarily
reflect a good choice for practical settings. We will revisit this question in
a later section.

From here on we let $B$ denote the number of update rounds of the algorithm:
\begin{equation}
B = \left|\Set{t > 1 \colon R_t < R_{t-1} }\right|\,.
\end{equation}
We can quantify the privacy guarantee of the algorithm in terms of this
parameter.
\begin{lemma}
Algorithm~\ref{fig:ladder} is $(\epsilon\sqrt{B}, O(\delta))$-differentially private.
\end{lemma}

\begin{proof}
For the purpose of its privacy analysis, the algorithm is equivalent to
the algorithm ``NumericSparse'' in~\cite{DworkR14} whose guarantees
follow from the sparse vector technique. The only difference in our algorithm is
that the threshold at each step varies. This difference is irrelevant for the
privacy analysis, since only the parameter~$B$ matters.

Since $\epsilon$ and $\delta$ are related multiplicatively through $\sigma,$ we
can absorb all constant factors appearing in the analysis of ``NumericSparse''
in the $O(\delta)$-term.
\end{proof}

Our goal is to invoke a ``transfer theorem'' that translates the privacy
guarantee of the algorithm into a generalization bound for the adaptive setting.
The following theorem due to Bassily et al.~\cite{BassilyNSSSU16} intuitively
shows that an $(\epsilon, \delta)$-differentially private algorithm is unable to
find a function that generalizes poorly.

\begin{theorem}[Theorem 7.2 in \cite{BassilyNSSSU16}]
\theoremlabel{bassily}
Let $\epsilon\in(0, 1/3), \delta\in (0, \epsilon/4),$ and $n\ge
\frac1{\epsilon^2}\log(4\epsilon/\delta)$. Let ${\cal M}$ be an
$(\epsilon,\delta)$-differentially private algorithm that, on input of a sample~$S$
of size $n$ drawn i.i.d.~from the population ${\cal D},$
returns a function $f\colon X\to[0, 1].$ Then,
\[
\Pr_{S, {\cal M}}\Set{\left|R_S(f)-R(f)\right| > 18\epsilon} < \frac\delta\epsilon\,.
\]
\end{theorem}

The original theorem is stated slightly differently. This version follows from
the fact that the empirical risk with respect to a bounded loss function has
``sensitivity'' $1/n$ in the terminology of~\cite{BassilyNSSSU16}.

Relevant to us is the following corollary.

\begin{corollary}
\corollarylabel{generalization}
Let $f_1,\dots,f_k$ be the functions encountered by the Shaky Ladder algorithm 
(\figureref{ladder}). Then, taking probability over both the sample~$S$
and the randomness of the algorithm, we have
\[
\Pr\Set{\max_{1\le t\le k}\left|R_S(f_t)-R(f_t)\right| > 18\epsilon\sqrt{B}} <
O\left(\frac{k\delta}\epsilon\right)\,.
\]
\end{corollary}
\begin{proof}
Let $\epsilon'=18\epsilon\sqrt{B}$ and $\delta'=O(k\delta/\epsilon).$
To apply \theoremref{bassily} we need to observe that the composition of the
Shaky Ladder algorithm with an arbitrary analyst (who does not otherwise have
access to the sample~$S$) satifies $(\epsilon', \delta')$-differential privacy at
every step of the algorithm. Hence, every function~$f_t$ is generated by an
$(\epsilon', \delta')$-differentially private algorithm so that the theorem
applies. The corollary now follows from a union bound over all~$k$ functions.
\end{proof}

\begin{lemma}
\lemmalabel{lberr}
Let $L_1,\dots,L_{3k+1}$ be all the Laplacian variables generated by our algorithm
and consider the maximum absolute value $L=\max_{1\le i \le k'} |L_i|.$ Then,
\[
\Pr\Set{
\lberr(R_1,\dots,R_k) >
18\epsilon\sqrt{B} + \lambda + 2L}
\le O\left(\frac{k\delta}\epsilon\right)
\]
\end{lemma}
\begin{proof}
In the comparison step of the algorithm at step~$t$, note that
\[
R_S(f_t)+\xi_t +\lambda + \xi
= R(f_t) + e,
\]
where $|e|\le 18\epsilon\sqrt{B} + \lambda + 2L.$ Here we used
\corollaryref{generalization}, as well as our bound on the Laplacian random
variables. Similarly, if we update $R_t$ at step~$t$, we have that
\[
|R_t - R(f_t)| \le 18\epsilon\sqrt{B} + L\,.
\]
Hence, we can think of our algorithm as observing the population risk of each
classifier up to the specified error bound. This implies, by induction, that the
estimates achieve the specified leaderboard error.
\end{proof}

We have the following tail bound for the quantity~$L$ that appeared in
\lemmaref{lberr}.

\begin{lemma}
\lemmalabel{Ltail}
For every $\beta>0,$
$\Pr\Set{L>\log(4k/\beta)\sigma}\le \beta\,.$
\end{lemma}
\begin{proof}
Note that $L$ is the maximum of at most $4k$ centered Laplacian random variables
with standard deviation~$\sigma.$ For a single such random variable, we have
\[
\Pr\{\left|\Lap(\sigma)\right|> t\sigma\}
= 2\int_{t\sigma}^\infty \frac1{2\sigma}\exp(-r/\sigma)\mathrm{d}r
=\int_{t}^\infty \exp(-u)\mathrm{d}u
= \exp(-t)\,.
\]
The claim now follows by applying this bound with $t=\log(4k/\beta)$ and taking a union bound over all $3k+1\le 4k$ Laplacian variables which $L$ is the maximum of.
\end{proof}

We also need to bound the number of update steps~$B.$ This is easy to do assuming we have a bound on~$L.$

\begin{lemma}
$\Pr\Set{B\le 4/\lambda \mid L \le \lambda/4} = 1.$
\end{lemma}
\begin{proof}
Assume that $L\le \lambda/4.$ This implies that whenever $t$ satisfies
\begin{equation}\label{update}
R_S(f_t) +\xi_t < R_{t-1} -\lambda + \xi,
\end{equation}
we must also have $R_S(f_t) < R_{t-1} - \lambda/2.$ Since $R_t = R_S(f_t)+\xi'_t,$ 
we also have $R_t < R_S(f_t)+\lambda/4.$ Therefore, $R_t < R_{t-1} -\lambda/4.$ In particular, we can have at most $4/\lambda$ rounds~$t$ for which the event~\eqref{update} occurs.
\end{proof}

\begin{theorem}
\theoremlabel{ub}
There is a constant $C>0$ such that with suitably chosen parameter settings the
Shaky Ladder algorithm (\figureref{ladder}) satisfies for any sequence of
adaptively chosen classifiers $f_1,\dots,f_k,$
\[
\Pr\Set{
\lberr(R_1,\dots,R_k)
> C\cdot\frac{\log(k/\beta)^{2/5}\log(kn/\beta)^{1/5}}{n^{2/5}}
}
\le \beta\,.
\]
\end{theorem}

\begin{proof}
Consider the event~${\cal G}$ that simultaneously
$L\le \log(4k/\beta)\sigma,$
and 
$\lberr(R_1,\dots,R_k) \le 18\epsilon\sqrt{B} + \lambda + 2L.$
Invoking our tail bounds from
\lemmaref{Ltail} and \lemmaref{lberr}, we have
that
\[
\Pr\Set{\cal G}\ge 1-O(k\delta/\epsilon) - \beta \ge 1 - O(\beta)\,.
\]
Here we used the definition of~$\delta$ and the fact that $\epsilon\ge 1/n.$

Proceeding under the condition that ${\cal G}$ occurs, we can plug in our parameter settings from \equationref{parameters} to verify that
\[
\lberr(R_1,\dots,R_k)
\le 18\epsilon\sqrt{B} + \lambda + 2L
\le O\left(\frac{\log(k/\beta)^{2/5}\log(kn/\beta)^{1/5}}{n^{2/5}}\right)\,.
\]
Rescaling $\beta$ to eliminate the constant in front of the error probability bound establishes the bound claimed in the theorem.
\end{proof}

\section{Connection to general adaptive estimation}
\sectionlabel{reduction}

In the general adaptive estimation setting, the adaptive analyst choose a
sequence of bounded functions $g_1,\dots,g_k\colon X\to[0,1]$ usually called
\emph{queries}. The algorithm must return estimates~$a_1,\dots,a_k$ in an
online fashion such that each estimate $a_k$ is close to the population
expectation~$\E_{\cal D}g_k.$ We will refer to algorithms in this setting as
\emph{general adaptive estimators} to distinguish them from \emph{leaderboard
algorithms} that we studied earlier.  The following definition of accuracy is
common in the literature. 

\begin{definition}
We say that a general adaptive estimator $\cal B$ is \emph{$(\alpha, \beta)$-accurate} on $n$ samples and $k$ queries if for every distribution over~$X,$ given $n$ samples from the distribution and adaptively chosen queries $g_1,\dots,g_k\colon X\to[0,1],$ the algorithm ${\cal B}$ returns estimates $a_1,\dots, a_k$ such that
$\Pr\Set{\max_{1\le t\le k} \left|\E_{\cal D} g_t - a_t\right|\le \alpha}\ge
1-\beta\,.$
\end{definition}

To bear out the connection with the leaderboard setting, we introduce an analogous definition for leaderboard error.

\begin{definition}
We say that a leaderboard algorithm $\cal A$ is \emph{$(\alpha, \beta)$-accurate} on $n$ samples and $k$~classifiers if for every distribution over $X\times Y$ and every bounded loss function, given $n$ samples and adaptively chosen sequence of classifiers $f_1,\dots,f_k\colon X\to Y,$ the algorithm ${\cal A}$ returns estimates $R_1,\dots, R_k$ such that
$\Pr\Set{\lberr(R_1,\dots,R_k) \le \alpha}\ge 1-\beta\,.$
\end{definition}

Given these definition, we can show a reduction from designing general adaptive estimators to designing leaderboard algorithms in the regime where the number of queries~$k$ is small.

\begin{theorem}
\theoremlabel{reduction}
Suppose there exists a leaderboard algorithm $\cal A$ that is $(\alpha/2, \beta)$-accurate on $n$ samples and $1/\alpha^2$ classifiers. Then, there exists a general adaptive estimator~$\cal B$ that is $(\alpha, \beta)$-accurate on $k=1/3\alpha$ queries. Moreover if ${\cal A}$ is computationally efficient, then so is ${\cal B}.$
\end{theorem}

\begin{proof}
Assume the existence of~${\cal A}$ and construct ${\cal B}$ as follows. Let~${\cal D}$ be the distribution over~$X$ for which ${\cal B}$ needs to be accurate.
Take the range~$Y=[0,1]$ and let the loss function be $\ell(y, y') = y.$ With this loss function, we can think of a query $g\colon X\to[0,1]$ as a classifier that satisfies $R_{\cal D}(g) = \E_{\cal D} g.$

At each step $1\le t\le k,$ the algorithm ${\cal B}$ receives a query $g_t$
from an adaptive analyst and has to use the algorithm ${\cal A}$ to answer the
query. The algorithm~$\cal B$ is described in \figureref{reduction}. Note that
all functions constructed in this procedure range in~$[0,1].$

Our first claim shows that if ${\cal A}$ has small leaderboard error, then the answers extracted from the above procedure are accurate.
\begin{claim}
If ${\cal A}$ has leaderboard error $\alpha/2,$ then $|a_t-\E_{\cal D} g_t|\le\alpha.$
\end{claim}
\begin{proof}
First note that by construction
\[
R(f_{t,i}) = c -\frac{i\alpha}2 + \frac12\E_{\cal D}g_t\,.
\]
By the definition of leaderboard error and our assumption, if $R(f_{t,i}) < c-\alpha/2,$ the algorithm ${\cal A}$ must output a value~$r_{t,i}$ that is lower than $c$ and moreover satisfies $|r_{t,i}-R(f_{t,i})|\le\alpha/3.$ By definition, $r_{t,i} = a_t/2 + c - i\alpha/2$ and therefore,
\[
r_{t,i}-R(f_{t,i}) = \frac{a_t}2 - \frac{\E_{\cal D}g_t}2.
\]
Hence,
\[
\left|a_t-\E_{\cal D}g_t\right|\le \alpha.
\]
\end{proof}

Our second claim ensures that we don't lower the threshold~$c$ too quickly, thus allowing ${\cal B}$ to answer sufficiently many queries.
\begin{claim}
If ${\cal A}$ has leaderboard error $\alpha/2,$ then 
the procedure we run for each function $g_t$ lowers the threshold $c$ by at most $3\alpha/2.$
\end{claim}
\begin{proof}
Observe that $R(f_{t,i+1}) \ge R(f_{t,i})-\alpha/2.$ In other words, the difference in risk of any two consecutive classifiers is bounded by~$\alpha/2.$ Hence, $r_{t, i+1}\ge r_{t, i}-3/\alpha/2.$ Therefore, the threshold~$c$ can decrease by at most~$3\alpha/2.$
\end{proof}

Assuming ${\cal A}$ has leaderboard error~$\alpha/2,$
the previous claim implies that the algorithm~${\cal B}$ can use the algorithm~${\cal A}$ for up to $k'=1/3\alpha$ queries before the threshold $c$ reaches~$0.$ The total number of classifiers that ${\cal B}$ gives to ${\cal A}$ is bounded by $1/\alpha^2.$ 
\end{proof}
It is natural to ask if the converse of the theorem is also true. Ideally, we
would like to have a result showing that a general adaptive estimator for few
queries implies a leaderboard algorithm for many queries. However, at this
level of generality it is not clear why there should be such an argument to
amplify the number of queries. Of course, by definition, we can say that a
general adaptive estimator for $k$ queries implies a leaderboard algorithm for
$k$ queries with the same accuracy.
\begin{figure}[h]
\setlength{\fboxsep}{2mm}
\begin{center}
\begin{boxedminipage}{\textwidth}

\noindent {\bf Input:} Data sets $S$ with $n=|S|,$ blackbox access to algorithm~${\cal A}.$

\noindent {\bf Algorithm ${\cal B}$:}

Given the query~$g_t,$ the algorithm ${\cal B}$ runs the following sequence of queries against~${\cal A}:$

\begin{itemize}
\item Set the threshold $c\in[0,1/2]$ to be the last value that ${\cal A}$ returned. If ${\cal A}$ has not previously been invoked, set $c=1/2.$
\item For $i=0$ to $i=1/\alpha-1:$
\begin{itemize}
\item Construct the function
$f_{t,i} = c + \frac12(g_t-i\alpha).$
\item Give the function $f_{t,i}$ to ${\cal A}$ and observe its answer $r_{t,i}.$
\item If $r_{t,i} < c - \alpha/2,$ put $a_t = 2(r_{t,i}-c+i\alpha/2)$ and stop. Else, continue.
\end{itemize}
\end{itemize}
\end{boxedminipage}
\end{center}
\vspace{-3mm}
\caption{
\figurelabel{reduction} Reduction from general estimation to leaderboard estimation.
}
\end{figure}

\subsection{Lower bounds for faithful algorithms}
\sectionlabel{faithful}

In this section, we prove a lower bound on a natural class of leaderboard algorithms that we call \emph{faithful}. It includes both of the algorithms proposed by Blum and Hardt, the Ladder and the parameter-free Ladder algorithm. Both of these algorithms are deterministic, but the class of faithful algorithms also includes many natural randomization schemes.

\begin{definition}
A leaderboard algorithm is \emph{faithful} if given a sample $S$ of size $n$
for every adaptively chosen sequence of models $f_1,\dots,f_k$ its estimates
$(R_1,\dots,R_k)$ satisfy with probability~$2/3$ for all $1<t\le k$ such that
$R_t < R_{t-1},$ we also have $|R_t - R_S(f_t)| \le \frac 1{2\sqrt{n}}$
\end{definition}
In words, given that the algorithm updated its estimate, i.e., $R_t < R_{t-1},$ the new estimate is likely close to the empirical risk of the $t$-th model. The constants in the definition are somewhat arbitrary. Other choices are possible. What matters is that the algorithm returns something close to the empirical risk with reasonably high probability whenever it gives feedback at all.

To prove a lower bound against faithful algorithms, we will invoke our connection with the general estimation setting.
\begin{definition}
\definitionlabel{general-faithful}
A general adaptive estimator is \emph{faithful} if given a sample $S$ of size
$n$ for every sequence of adaptively chosen function~$g_1,\dots,g_k$ its
estimates $(a_1,\dots,a_k)$ satisfy with probability $2/3,$
$\forall t\colon \left|a_t - \frac1n\sum_{x\in S}g_t(x)\right|\le
\frac1{2\sqrt{n}}.$
\end{definition}

The reduction we saw earlier preserves faithfulness.

\begin{lemma}
\lemmalabel{faithful}
If ${\cal A}$ is a faithful leaderboard algorithm, then the algorithm~${\cal B}$ resulting from the reduction in~\figureref{reduction} is a faithful general adaptive estimator.
\end{lemma}

We can therefore obtain a lower bound on faithful leaderboard algorithms by proving one against faithful general adaptive estimators.

\begin{theorem}
\theoremlabel{faithful}
No faithful general adaptive estimator is $(o(\sqrt{k/n}), 1/4)$-accurate on $n$
samples and $k\le n$ queries.
\end{theorem}

\begin{proof}
Set up the distribution~${\cal D}$ over $X\times Y$ with the label set
$Y=\{0,1\}$ such that the label $y$ is uniformly random conditional on any
instance~$x\in X.$
Fix a general adaptive estimator~${\cal B}$ that gets a sample~$S$ of
size $n$ drawn from~${\cal D}.$ We need to show that the estimator~${\cal B}$
cannot be $(o(\sqrt{k/n}), 1/4)$-accurate. 
To show this claim we will analyze the following procedure (majority
attack):
\begin{itemize}
\item Pick $k\le n$ random functions $f_1,\dots,f_k\colon X\to\{0,1\}$. 
\item Let $a_i = R_S(f_i)$ be the empirical risk of $f_i$ with respect to the
$0/1$-loss. Further, let $\hat a_i$ be the answer from the general adaptive
estimator on the query $g_i(x,y) = \mathbb{I}\Set{f_i(x) \ne y}.$ 
\item Consider the index set 
$I = \Set{i \colon \hat a_i < 1/2 - 1/\sqrt{n}}.$
\item Let $f=\mathrm{maj}_{i\in I} f_i$ be the pointwise majority function of
all functions in~$I$. That is $f(x)$ is the majority value among $f_i(x)$ with
$i\in I.$
\item Ask ${\cal B}$ to estimate the $0/1$-loss of~$f,$ i.e., submit the query
$g^*(x,y) = \mathbb{I}\Set{f(x)\ne y}.$
\end{itemize}
Note that $\E_{x,y\sim\cD} g^*(x,y) = R(f)$ and hence it remains to analyze the
difference between the risk and empirical risk of~$f.$
\begin{claim}
\claimlabel{majority-risk}
$R(f)=1/2.$
\end{claim}
\begin{proof}
This is true for any function $f\colon X\to Y$ 
given the way we chose the distribution over $X\times Y.$
\end{proof}
We claim that 
the empirical risk is bounded away from~$1/2$ by
$\Omega(\sqrt{k/n})$ with constant probability. 
A similar claim appeared in~\cite{BH15} without proof.
\begin{claim}
\claimlabel{majority-empirical}
Assume $k\le n.$ Then, with probability~$1/3,$
\[
R_S(f) \le 1/2 - \Omega\left(\sqrt{k/n}\right) - O\left(1/\sqrt{n}\right)\,.
\]
\end{claim}
\begin{proof}
Following \definitionref{general-faithful}, condition on the event that
for all~$t\in[k],$ we have $|a_t-\hat a_t|\le 1/2\sqrt{n}.$ 
By the definition, this even occurs with
probability~$2/3.$ 
Under this condition all $i\in I$ satisfy $a_i < 1/2 - 1/2\sqrt{n}.$
Furthermore, we claim that $|I|\ge \Omega(k)$ with probability $2/3.$ This
follows because $\Pr\{a_i < 1/2 - 1/\sqrt{n}\} = \Omega(1).$ In particular both
events occur with probability at least~$1/3.$
Let 
\[
\epsilon_i = \Pr_{(x,y)\in S}\Set{f_i(x) = y} - 1/2
\] 
be the advantage over
random of $g_i$ in correctly labeling an element of~$S.$ By definition of
$\epsilon_i,$ we must have that $\epsilon_i > 1/2\sqrt{n}$ for all $i\in I.$
We will argue that this advantage over random is amplified by the majority vote.

Let $Z_i$ be the indicator of the event that $f_i(x) = y$ for random $(x,y)\in
S.$ For ease of notation rearrange indices such that $I=\{1,2,\dots, m\},$
where $m=\Omega(k)$ as argued earlier.
We know that $Z_i$ is Bernoulli with parameter $1/2+\epsilon_i$ where by
construction $\epsilon_i\ge 1/2\sqrt{n}.$ 
Let $Z$ be the indicator of the event that $f(x) = y.$
Let $\epsilon=1/2\sqrt{n}$ and observe that $\epsilon\le1/\sqrt{m}$ since $k\le
n.$ Therefore,
\begin{align*}
\Pr\{Z=1\}
& \ge \frac12\Pr\Set{\sum_{i=1}^{m} Z_i > m/2}\\
& \ge  \Pr\Set{\mathrm{Binomial}(m, 1/2+\epsilon) > m/2}\\
& \ge \frac12 + \Omega\left(\sqrt{m}\epsilon\right) - O\left(1/\sqrt{m}\right) \tag{\claimref{anti}, using $\epsilon<1/\sqrt{m}$}\\
& = \frac12 + \Omega\left(\sqrt{k/n}\right) - O\left(1/\sqrt{n}\right)\,.
\end{align*}
The claim now follows, since~$R_S(f)=1-\Pr\{Z=1\}.$
\end{proof}
Taking \claimref{majority-risk} and \claimref{majority-empirical} together, we
have that $R(f)-R_S(f)\ge\Omega(k/n)-O(1/\sqrt{n}),$ with probability~$1/3.$
In particular, when $k=\omega(1),$ this shows that the estimator~${\cal B}$ is
not $(o(\sqrt{k/n}), 1/4)$-accurate. For $k=O(1),$ the same claim follows from
a standard variance calculation. 
\end{proof}

The previous theorem
implies that faithful leaderboard algorithms cannot have leaderboard error
better than~$n^{1/3}.$

\begin{corollary}
\corollarylabel{faithful}
No faithful leaderboard algorithm is $(\alpha,
\beta)$-accurate on $n$ samples and $k$ queries for any $\alpha=
k^{o(1)}/n^{1/3-c},$ $\beta=1-o(1)$ and constant~$c>0.$
\end{corollary}

\begin{proof}
Combine our lower bound from~\theoremref{faithful} with the reduction
in~\theoremref{reduction}. By~\lemmaref{faithful}, faithfulness is preserved
and hence we get the stated lower bound.
\end{proof}

\section{Experiments with a shifted majority attack}
\sectionlabel{attack}
\sectionlabel{experiments}

The attack implicit in \corollaryref{faithful} corresponds to what we will call the
\emph{shifted majority attack}. To understand the idea, we briefly review the
\emph{Boosting attack} from~\cite{BH15}. In this procedure, the analyst
first asks $k$ random queries (thought of as vectors in $\{0,1\}^n$, one
binary label for each point in the holdout set), and then selects the ones
that have error ($0/1$-loss) less than~$1/2.$ Note that the expected loss is
$1/2.$ Among these selected queries, the analyst computes a coordinate-wise
majority vote, resulting in a final output vector $\hat y\in\{0,1\}^n.$ Blum
and Hardt observed that this output vector has expected error
$1/2-\Omega(\sqrt{k/n}),$ with respect to the true holdout
labels~$y\in\{0,1\}^n.$ Despite the fact that the vector setup is a slight
simplification of the actual formal framework we have, this idea carries over
to our setting by replacing random vectors with random functions. We will
refer to this procedure as \emph{majority attack}.

The majority attack has the property that when run against the Ladder
algorithm, the analyst quickly stops receiving new feedback. Newly chosen
random functions are increasingly unlikely to improve upon the error of previous
functions. Our procedure in~\figureref{reduction}, however, shows how to
offset the queries in such a way that the analyst continues to receive as much
feedback as possible from the algorithm. In theory, this requires knowledge
about the underlying distribution (which is fine for the purpose of proving
the theorem). In reality, we can imagine that there may be a subset of the
domain on which the classification problem is easy so that the analyst knows a
fraction of the labels with near certainty. The analyst can then use this
``easy set'' to offset the functions as required by the attack. This leads to
what we call the \emph{shifted majority attack}.

\paragraph{Setup.}
Rather than running the shifted majority attack, we will run the majority attack
for a varying number of queries~$k.$ The reason for this setup is that there is
no canonical parameter choice for the implementation of the Ladder algorithm, or
the Shaky Ladder. In particular, the number of queries that can be answered
using the shifting idea is closely related to the inverse of the step size
parameter. It is therefore more transparent to leave the number of queries as a
parameter that can be varied. \sectionref{implementation} contains a reference
implementation of the majority attack that we experiment with.

The primary purpose of our experiments is to understand in simulation the
effect of adding noise to the feedback of the leaderboard algorithm.

\paragraph{Observations.}
\figureref{varyqueries} shows that even a small amount of Gaussian
noise (e.g., standard deviation $\sigma=3/\sqrt{n}$) mostly neutralizes the
majority attack that is otherwise very effective against the standard
Ladder algorithm.  We note in passing that the \emph{parameter-free} Ladder
algorithm~\cite{BH15} only reveals more feedback than the Ladder algorithm. As
such it fares even more poorly than the Ladder algorithm under the shifted
majority attack. 

\figureref{varynoise} consolidates the observation by showing the effect of
varying noise levels. There appears to be a sweet spot at $3$ standard
deviations, where much of the harm of the shifted majority attack is
neutralized, while the amount of noise added is still small as a function
of~$n.$ In particular, in simulation it appears that less noise is necessary than
our theorem suggests.

\begin{figure}
\begin{center}
\includegraphics[width=0.48\textwidth]{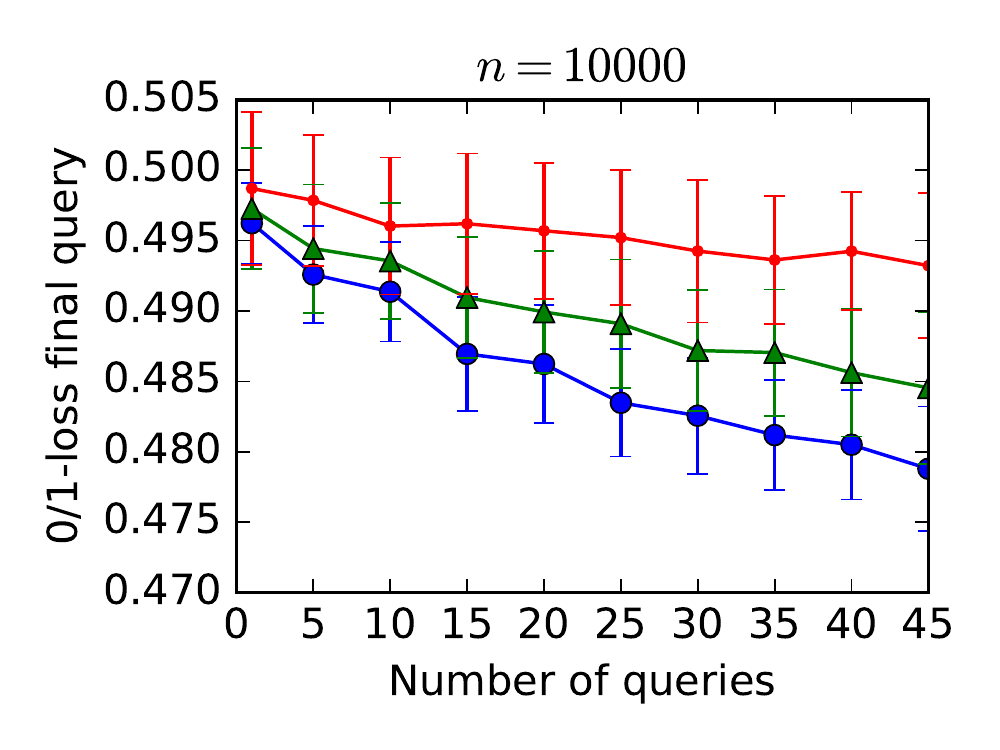}
\includegraphics[width=0.48\textwidth]{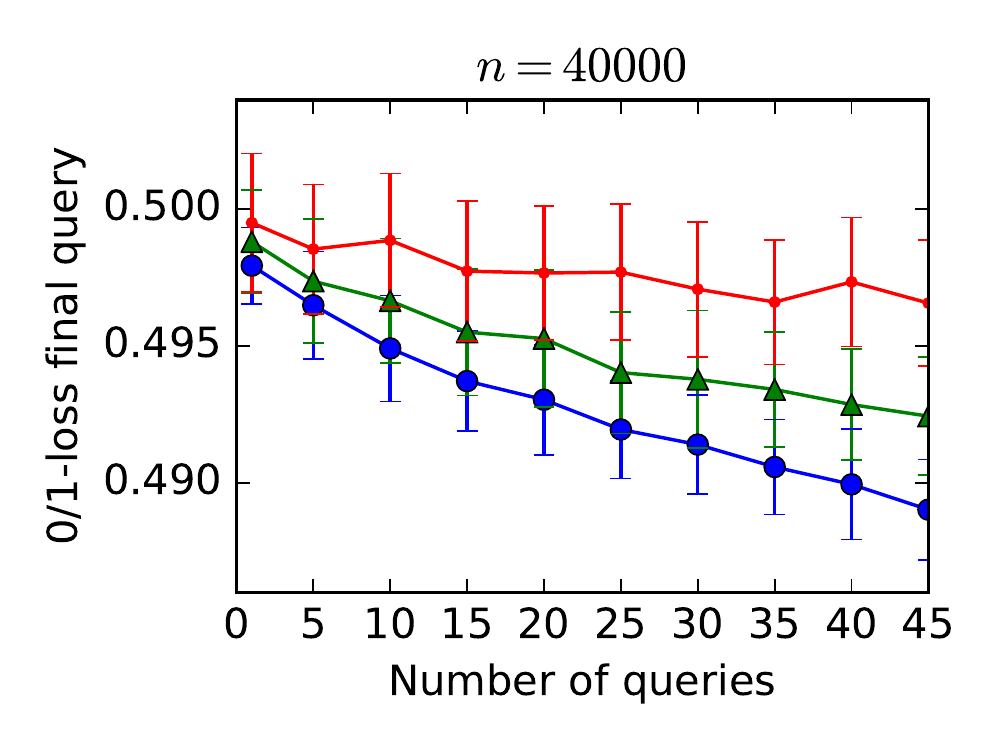}
\vspace{-5mm}
\includegraphics[width=0.48\textwidth]{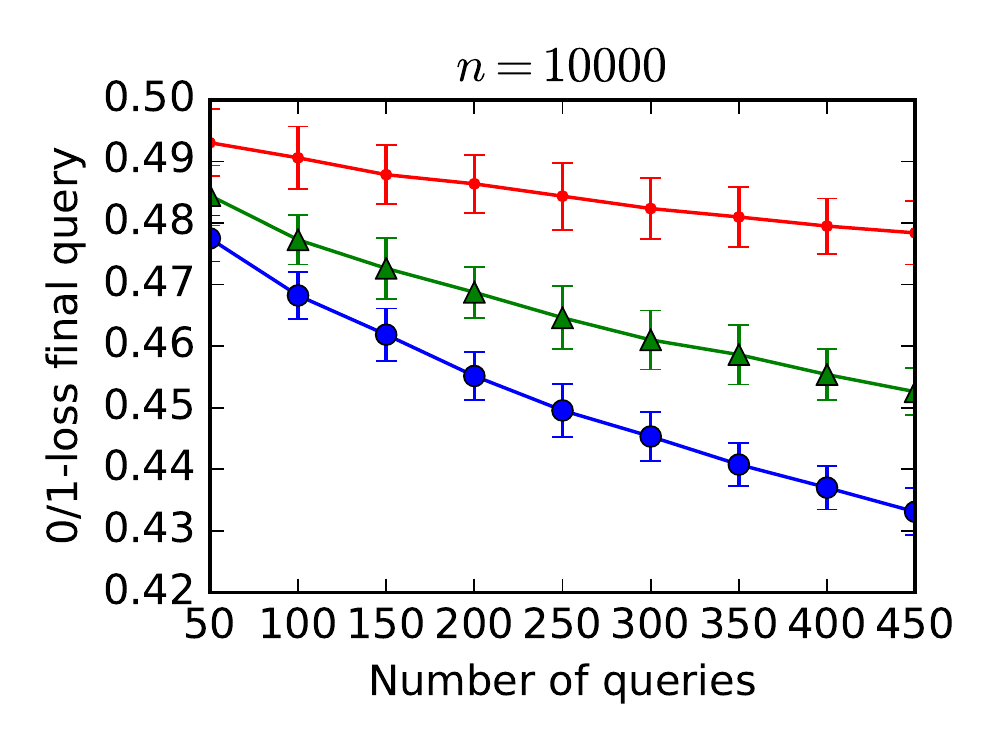}
\includegraphics[width=0.48\textwidth]{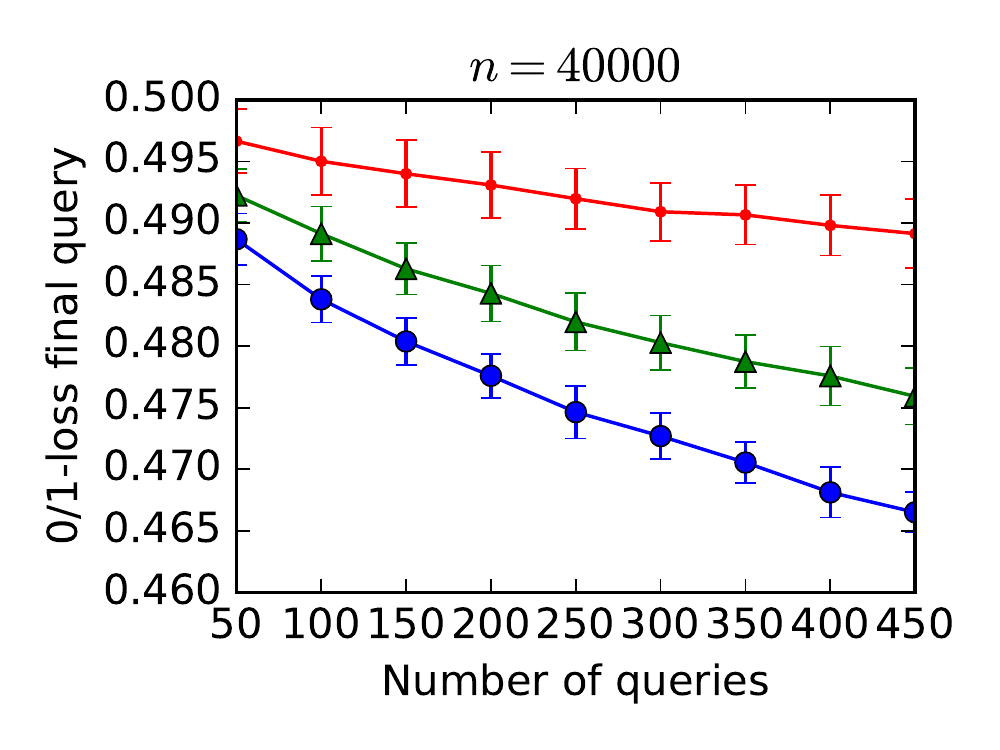}
\end{center}
\caption{Varying number of queries for different noise levels. Bottom line:
no noise. Middle line: $1/\sqrt{n}.$ Top line: $3/\sqrt{n}.$ Error bars
indicate standard deviation across $100$ independent repetitions.}
\figurelabel{varyqueries}
\end{figure}

\begin{figure}
\begin{center}
\includegraphics[width=0.48\textwidth]{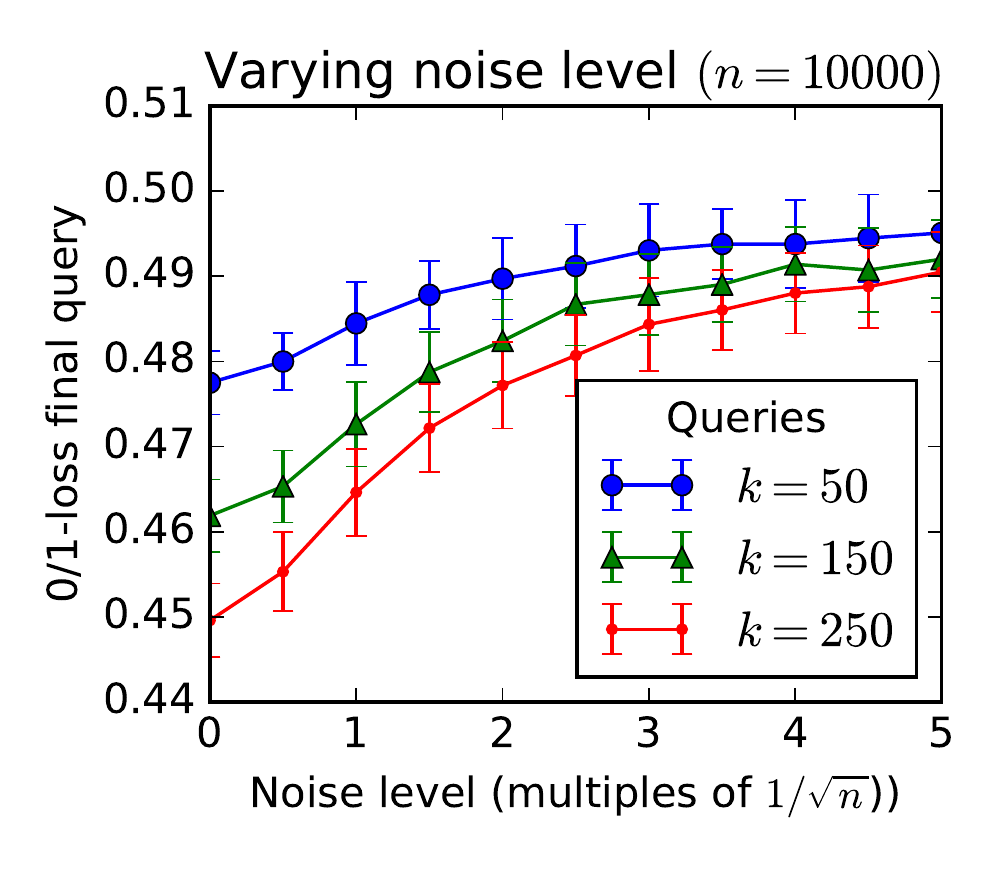}
\includegraphics[width=0.48\textwidth]{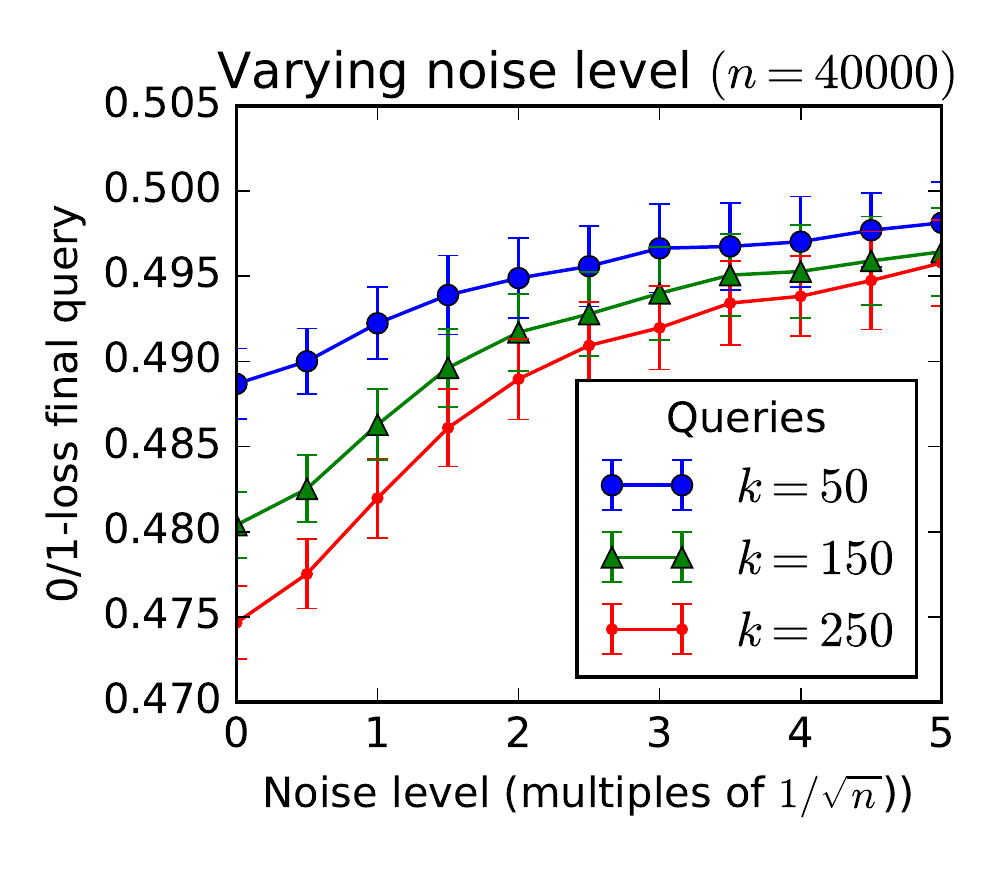}
\vspace{-5mm}
\end{center}
\caption{Varying noise level for different number of queries. Error bars
indicate standard deviation across $100$ independent repetitions.}
\figurelabel{varynoise}
\end{figure}

\section{Conclusion and open problems}

We saw a new algorithm with leaderboard error $O(n^{-0.4}).$ This upper bound
lies strictly between the two more natural bounds of $O(n^{-1/3})$ and
$O(n^{-1/2}).$ If experience from online and Bandit learning is any guide, the
new upper bound might suggest that there is hope of attaining the tight
$O(n^{-1/2})$ error rate. This possibility is further supported by the fact that
the majority attack we saw in~\sectionref{experiments} is quite sensitive to
noise on the order of~$O(n^{-1/2}).$ This leads us to conjecture that
$O(n^{-1/2})$ might in fact be the right answer. However, in light of our
connection between the general adaptive estimation setting and leaderboard
error, such a conjecture can now be refuted by stronger lower bounds for the
general adaptive estimation setting. It is unclear if more sophisticated lower
bounding techniques based on Fingerprinting codes as used 
in~\cite{HardtU14,SteinkeU14} could be used to obtain stronger lower bounds in
the small number of query regime ($k\ll n$).

\section*{Acknowledgments}

Many thanks to Avrim Blum, Yair Carmon, Roy Frostig, and Tomer Koren for insightful
observations and suggestions at various stages of this work.

\pagebreak
\bibliographystyle{moritz}
\bibliography{moritz}
\vfill

\pagebreak
\appendix
\section{Anti-concentration inequality for the Binomial distribution}

\begin{claim}
\claimlabel{anti}
Let $0 <\epsilon \le 1/\sqrt{m}.$ Then,
\[
\Pr\Set{\mathrm{Binomial}(m, 1/2+\epsilon) > m/2}
\ge \frac12 + \Omega\left(\sqrt{m}\epsilon\right) - O\left(1/\sqrt{m}\right)\,.
\]
\end{claim}
\begin{proof}
Put $p=1/2+\epsilon$ and $q=1-p.$
On the one hand, for the given upper bound on~$\epsilon,$ the Berry-Esseen
theorem implies the normal approximation
\[
\Pr\Set{\mathrm{Binomial}(m, 1/2+\epsilon) > m/2}
\ge
\Pr\Set{\mathrm{N}(mp, mpq) > mp - \epsilon m}- O\left(1/\sqrt{m}\right)\,.
\]
On the other hand,
\begin{align*}
\Pr\Set{\mathrm{N}(mp, mpq) > mp - \epsilon m}
= \Pr\Set{\mathrm{N}(0, pq) > - \epsilon \sqrt{m}}
\ge \frac12 + \Omega\left(\epsilon\sqrt{m}\right)\,.
\end{align*}
In the last step, we used the our upper bound on $\epsilon,$ which ensures
that $\epsilon\sqrt{m} \le 1.$ Noting that $pq\ge\Omega(1),$ the last step now
follows from the fact that the density of $\mathrm{N}(0, pq)$ is lower bounded
by a constant in the interval $[-\epsilon\sqrt{m}, 0].$

Putting the two observations together we get
\[
\Pr\Set{\mathrm{Binomial}(m, 1/2+\epsilon) > m/2}
\ge \frac12
+ \Omega\left(\epsilon\sqrt{m}\right)
- O\left(1/\sqrt{m}\right)
\,.
\]
\end{proof}

\section{Reference implementation for majority attack}
\sectionlabel{implementation}

For definedness, we include a reference implementation of the majority attack
used in our experiments.

\begin{lstlisting}[language=Python,basicstyle=\ttfamily]
import numpy as np

def majority_attack(n, k, sigma=None):
    """Run majority attack and report resulting bias."""
    hidden_vector = 2.0 * np.random.randint(0, 2, n) - 1.0
    queries = 2.0 * np.random.randint(0, 2, (k, n)) - 1.0
    answers = queries.dot(hidden_vector)/n
    if sigma:
        answers += np.random.normal(0, sigma, k)
    positives = queries[answers > 0., :]
    negatives = queries[answers <= 0., :]
    weighted = np.vstack([positives, -1.0*negatives])
    weights = weighted.T.dot(np.ones(k))
    final = np.ones(n)
    final[weights < 0.] = -1.0
    return np.mean(final != hidden_vector)
\end{lstlisting}

\end{document}